\documentclass{article}

\usepackage{arxiv}

\usepackage[utf8]{inputenc} 
\usepackage[T1]{fontenc}    
\usepackage{hyperref}       
\usepackage{url}            
\usepackage{booktabs}       
\usepackage{amsfonts}       
\usepackage{nicefrac}       
\usepackage{microtype}      
\usepackage{lipsum}

\usepackage{amssymb}
\usepackage{amsthm}

\usepackage{graphicx}
\usepackage{amsmath}
\usepackage[vlined,linesnumbered,ruled,titlenotnumbered]{algorithm2e}
\usepackage{subcaption}
\usepackage{url}
\urldef{\mailsa}\path||

\newtheorem{theorem}{Theorem}

\newtheorem{lemma}{Lemma}
\newtheorem{corollary}{Corollary}

\DeclareMathOperator*{\argmax}{arg\,max}

\title{Analysis of Baseline Evolutionary Algorithms for the Packing While Travelling Problem}

\author{
    Vahid Roostapour\thanks{Corresponding author}\\
    School of Computer Science\\
    The University of Adelaide\\
    Adelaide, SA, Australia\\
    \texttt{vahid.roostapour@adelaide.edu.au}\\
    \AND
    Mojgan Pourhassan\\
    School of Computer Science\\
    The University of Adelaide\\
    Adelaide, SA, Australia\\
    \texttt{mojgan.pourhassan@adelaide.edu.au}\\
   \And
    Frank Neumann\\
    School of Computer Science\\
    The University of Adelaide\\
    Adelaide, SA, Australia\\
    \texttt{frank.neumann@adelaide.edu.au}\\
}

\begin{document}
\maketitle

\begin{abstract}
The performance of base-line Evolutionary Algorithms (EAs) on combinatorial problems has been studied rigorously. 
From the theoretical viewpoint, the literature extensively investigates the linear problems, while the theoretical analysis of the non-linear problems is still far behind.
In this paper, variations of the Packing While Travelling (PWT) -- also known as the non-linear knapsack problem -- are studied as an attempt to analyse the behaviour of EAs on non-linear problems from theoretical perspective.
We investigate PWT for two cities and $n$ items with correlated weights and profits, using single-objective and multi-objective algorithms. Our results show that RLS\_swap, which differs from the classical RLS by having the ability to swap two bits in one iteration, finds the optimal solution in $O(n^3)$ expected time. We also study an enhanced version of GSEMO, which a specific selection operator to deal with exponential population size, and prove that it finds the Pareto front in the same asymptotic expected time. In the case of uniform weights, (1+1)~EA is able to find the optimal solution in expected time $O(n^2\log{(\max\{n,p_{\max}\})})$, where $p_{\max}$ is the largest profit of the given items. We also perform an experimental analysis to complement our theoretical investigations and provide additional insights into the runtime behavior.
\end{abstract}


\section{Introduction}
Evolutionary algorithms (EAs) are general problem solvers which have been widely applied to many combinatorial optimisation and engineering problems~\cite{DBLP:books/sp/chiong12}. 
Similar to other bio-inspired algorithms, EAs aim to improve the current solution iteratively by imitating the evolution process in nature.

Random decision making plays an essential role in EAs' basic operations, such as selection and mutation. The area of runtime analysis aims to theoretically investigate the performance of EAs as a specific class of randomised algorithms and many significant contributions have been achieved along this 20 year old path~\cite{DBLP:journals/gpem/Popovici14,DBLP:series/ncs/Jansen13,neumann2010bioinspired}. Remarkable attention has been devoted to the baseline evolutionary algorithms, such as RLS, (1+1)~EA and GSEMO, since they can present a general insight into the performance of more complicated EAs. Interestingly, it has been proven that they perform efficiently in many cases~\cite{droste2002analysis,neumann2010bioinspired}. These analyses help to understand the performance of more complicated algorithms facing realistic environments.

Following Droste et al., who considered the performance of (1+1)~EA on a linear class of problems in 1998~\cite{DBLP:journals/ec/DrosteJW98}, many studies also considered linear functions~\cite{DBLP:journals/tcs/DoerrK15, Witt:2018:DCW:3205455.3205581}. Friedrich et al. proved that (1+1)~EA finds the optimal solution for any linear pseudo-boolean function under the uniform constraint in $O(n^2\log{(B\cdot p_{\max})})$, where $B$ is the constraint and $p_{\max}$ is the maximum profit of the given items~\cite{friedrich2018analysis}. More recently, Neumann et al. improved this bound to $O(n^2\log{B})$ expected time and to $O(n^2\log{n})$ with high probability~\cite{FrankGECCO}. They also proved a tight bound of $\Theta(n^2)$ for RLS.

While a number of simple nonlinear problems, such as trap functions and LeadingOnes, have been deeply studied from theoretical and practical aspects~\cite{DBLP:conf/ppsn/AntipovD18,DBLP:journals/tec/NijssenB03,giessen2016robustness}, the literature is not equally rich for investigations of EAs on non-linear functions. Specifically, multi-component problems have gained a lot of attention during recent years due to their appearance in many real-world applications~\cite{Bonyadi2019}. Multi-component problems consist of different components integrating with each other in a way that it is essential to deal with all the components simultaneously to find good quality solutions. Examples include the Generalised Minimum Spanning Tree problem (GMST) and the Generalised Travelling Salesman Problem (GTSP) in which clusters of nodes exist and the problem is to chose one node from each cluster and solve the MST and TSP on the selected nodes, respectively. These problems can be decomposed into a node selection component and a MST/TSP component, and hierarchical approaches for solving them are investigated theoretically in~\cite{corus2016parameterised} and~\cite{pourhassan2018theoretical}.

The Travelling Thief Problem (TTP), introduced by Bonyadi et al., is another example of an $\mathcal{NP}$-hard multi-component problem that combines the travelling salesman problem and the knapsack problem as the components~\cite{DBLP:conf/cec/BonyadiMB13}. In this problem, there are a number of items distributed between the nodes of a graph and the goal is to find a Hamiltonian cycle and choose the items from the visited nodes such that the benefit function is maximised. The integration of these components in a non-linear objective function is done in a way that optimising one component does not necessarily result in a near optimal solution. Using the same formulation, Polyakovskiy and Neumann introduced the Packing While Travelling problem (PWT), which is actually the packing version of TTP with a fixed travelling path~\cite{DBLP:journals/eor/PolyakovskiyN17}. A vehicle with a specified capacity wants to visit all the nodes in a specific order, pick the items such that the benefit function is maximised, and the capacity constraint is not violated. The benefit function has a direct relationship with the profit of packed items while the relationship with velocity is inverse. The velocity between each of the two cities is determined by the weight of the vehicle such that the more weight the vehicle carries, the less velocity it has. They prove that this problem, for even two cities, is also $\mathcal{NP}$-hard. Later, in~\cite{DBLP:conf/algocloud/NeumannPSSW18}, Neumann et al. presented an exact dynamic programming approach for PWT. They also proved that there is no polynomial time algorithm with constant approximation ratio for PWT problem unless $\mathcal{P}=\mathcal{NP}$ and presented an FPTAS for maximising the objective value over the baseline travel cost in which the vehicle travels the path empty. 

In this paper, we theoretically analyse the performance of three baseline evolutionary algorithms, RLS\_swap, (1+1)~EA and GSEMO, on variations of the packing while travelling problem and prove upper bounds for the expected running times. We prove that for the instances with correlated weights and profits, RLS\_swap and GSEMO find the optimal solution in expected time $O(n^3)$. Furthermore, we consider the instances with uniform weights and prove that (1+1)~EA finds the optimal solution in expected time $O(n^2\log(\max\{n,p_{\max}\})$. We also investigate the performance of these algorithms in addition to two other multi-objective algorithms, which are introduced in section~\ref{sec:practical}, from the experimental point of view.

The rest of the paper is organised as follows. Section~\ref{sec:priliminaries} presents the detailed definition of considered PWT problem and the algorithms used in this study. In Section~\ref{sec:theoretical}, we investigate the problem theoretically for correlated weights and uniform weights in Sections~\ref{sub:Correlated} and~\ref{sub:uniform}, respectively. Our experimental analyses are presented in Section~\ref{sec:practical}, followed by a conclusion in Section~\ref{sec:conclusion}.

\section{Preliminaries}\label{sec:priliminaries}
In this section we present the definition of Packing While Travelling problem and the details of algorithms we analyse in this paper.

\subsection{Problem Definition}
The general PWT problem proposed in 2016 by Polyakovskiy and Neumann can be seen as a TTP with a fixed path~\cite{DBLP:conf/cpaior/PolyakovskiyN15}. Given a set of $m+1$ ordered cities, distances $d_i$ from city $i$ to $i+1$ ($1\leq i\leq m$) and a set of items $N=\cup_{i=1}^m{N_i}$ distributed over first $m$ cities such that city $1\leq i\leq m$ contains items $N_i$. Let $|N_i|=n_i$ denote the number of items in city $i$. Positive integers profit $p_{ij}$ and weight $w_{ij}$ are assigned to each item $e_{ij} \in N_i$, $1\leq j\leq n_i$.
Path $M = (1,2,\dots,m+1)$ is travelled by a vehicle with velocity $v=[v_{\min},v_{\max}]$ and capacity $C$. A solution vector $s=(x_{11}x_{12}\dots x_{1{n_1}}\dots x_{m{n_m}})$ represents a set of selected items $S\subseteq N$ such that variable $x_{ij}\in\{0,1\}$ indicates whether item $e_{ij}$ is selected or not. Let $W(s)$ denote the sum of weights of items in $s$. As such, $s$ is feasible if $W(s)\leq C$. Let $$P(s) = \sum_{i=1}^m{\sum_{j=1}^{n_i}{p_{ij}x_{ij}}}$$ be the total profit and $$T(s)=\sum_{i=1}^m\frac{d_i}{v_{\max}-\nu\sum_{k=1}^i\sum_{j=1}^{n_k}w_{kj}x_{kj}}\text{,}$$ where $\nu = \frac{v_{\max}-v_{\min}}{C}$ is a constant, be the total travel time for vehicle carrying items of $s$. The denominator of $T(s)$ is such that picking an item in city $i$ only affects the time to travel from city $i$ to the end. Thus, the benefit value of $s$ is computed as$$B(s)=P(s)-R\cdot T(s)\text{,}$$ where $R$ is a given renting rate.
The aim is to find a feasible solution $s^*=\max_{s\in\{0,1\}^n}B(s)$.

The effect of different values for $R$ on the benefit function has been considered previously~\cite{DBLP:conf/gecco/WuPN16}. Generally speaking, $R=0$ changes PWT to the 0-1 knapsack problem while a larger $R$ forces the optimal solution to pick fewer items. 
This problem is proven to be $\mathcal{NP}$-hard by reducing the subset sum problem to the decision variant of unconstrained PWT~\cite{DBLP:journals/eor/PolyakovskiyN17}.

In this paper we consider another version of this problem where there are only two cities and $n$ items that are located in the first city. In addition, the weights of the given items are favourably correlated with the profits, i.e. for any two item $e_i$ and $e_j$, $p_i> p_j$ implies $w_i< w_j$. Hence, for a bit string solution $s=(x_1\cdots x_n)\in \{0,1\}^n$, we have 

\begin{align}\nonumber
B(s) = \sum_{i=1}^{n}{x_i\cdot p_i} - \frac{Rd}{v_{\max}-\nu \sum_{i=1}^n{x_i\cdot w_i}}\text{,}
\end{align}
where $d$ is the distance between the two cities.

We assume there are no items with the exact same weight and profit, however, all the proofs can be extended to include these cases and achieve the same results. Moreover, without loss of generality, in the rest of the paper we assume that items are indexed in a way that $p_1\geq p_2\cdots \geq p_n$ and $w_1\leq w_2\cdots \leq w_n$.
We also consider another version of the problem in which all the weights are uniform and equal to one.
\subsection{Algorithms}
In this paper we study the behaviour of three algorithms. The first one, described in Algorithm~\ref{alg:RLSswap}, is a Random Local Search (RLS) variant called RLS\_swap, which is able to do a swap (flip a zero bit and a one bit simultaneously), in addition to the usual one-bit flip. This modification of the classical RLS has been previously considered for the MST problem in~\cite{DBLP:journals/tcs/NeumannW07}. In each iteration, if the current solution is all zeros or ones, it flips a randomly chosen bit of the solution. Otherwise, with probability of $1/2$, it either does a one-bit flip as described or chooses a one and a zero uniformly at random and flips both of them. The generated offspring replaces the current solution if it is at least as good as its parent with respect to the fitness function. This swap mutation is added in RLS\_swap because there are some situations when optimising PWT in which no one-bit flip is able to pass the local optima. 
\begin{algorithm}[t]
	Choose $s\in\{0,1\}^n$ uniformly at random\;
	Let $|s|_1$ denote the number of items in $s$\;
    \While {stopping criterion not met}{
    $p\leftarrow$ a random real number in $[0,1]$\;
    \eIf{$|s|_1=0\lor|s|_1=n\lor p<1/2$}{
    Create $s^\prime$ by flipping a randomly chosen bit of $s$ \;
    }{
    
    Create $s^\prime$ by flipping a randomly chosen zero bit and a randomly chosen one bit of $s$ \;
    }

    \If{$F(s^\prime) \geq F(s)$}{$s \leftarrow s^\prime$\;}
    }
    \caption{RLS\_swap}\label{alg:RLSswap}
    
\end{algorithm}

Another algorithm we consider is (1+1)~EA (Algorithm~\ref{alg:1+1ea}). This algorithm flips each bit of the current solution with probability of $1/n$ in each mutation step. Similar to RLS\_swap, it compares the parent and the offspring and picks the better one for the next generation.

\begin{algorithm}[t]
	Choose $s\in\{0,1\}^n$ uniformly at random\;
    \While {stopping criterion not met}{
    Create $s^\prime$ by flipping each bit of $s$ independently with probability of ${1/n}$\;
    \If{$F(s^\prime) \geq F(s)$}{$s \leftarrow s^\prime$\;}
    }
    \caption{(1+1)~EA}\label{alg:1+1ea}
    
\end{algorithm}
In RLS\_swap and (1+1)~EA, which are single-objective algorithms, the comparisons between solutions is based on the fitness function $$F(s) = (q(s),B(s))$$ where $q(s) = \min\{C-w(s),0\}$. According to $q(s)$, $s$ is infeasible if and only if $q(s)<0$ and the absolute value of $q(s)$ denotes the amount of constraint violation. The goal is to maximise $F(s)$ with respect to lexicographical order, i.e. $s_1$ is better than $s_2$ ($F(s_1)\geq F(s_2)$) if and only if $\left(q(s_1)>q(s_2)\right) \lor \left(q(s_1)=q(s_2)\land B(s_1)\geq B(s_2)\right) $. This implies that any feasible solution has better fitness than any infeasible solution. Moreover, between two infeasible solutions, the one with smaller constraint violation is better. 

We also consider PWT with a multi-objective algorithm using a variant of GSEMO, which uses a specific selection function to deal with the exponential size of the population (Algorithm~\ref{alg:GSEMO}). Neumann and Sutton suggested this version of GSEMO for the Knapsack Problem (KP) with correlated weights and profit to avoid an exponential population size~\cite{DBLP:conf/ppsn/0001S18}. We use the same approach since PWT easily changes to KP when $R=0$. As the objectives, we use the weight function ($W(s)$) and the previously defined fitness function ($F(s)$). The aim is to minimise $W(s)$ while maximising $F(s)$. Between two solutions $s_1$ and $s_2$, we say $s_1$ (weakly) dominates $s_2$, denoted by $s_1\succeq s_2$, if and only if $W(s_1)\leq W(s_2) \land F(s_1)\geq F(s_2)$. The dominance is called strong, denoted by $s_1\succ s_2$, when at least one of the inequalities strictly holds. Note that based on this definition, similar to the single-objective fitness function, each feasible solution dominates all infeasible solutions and an infeasible solution closer to the constraint bound dominates the more distant ones. 
\begin{algorithm}[t]
	Choose $s\in\{0,1\}^n$ uniformly at random\;
	$P\leftarrow\{s\}$\;
    \While {stopping criterion not met}{
    Let $P_i=\{s\in P\mid |s|_1 = i\},0\leq i\leq n$ and $I = \{i\mid P_i\neq\emptyset\}$\;
    Choose $j\in I$ uniformly at random\;
    $s\leftarrow \argmax\{B(x)\mid x\in P_j\}$\;
    Create $s^\prime$ by flipping each bit of $s$ independently with probability of ${1/n}$\;
    \If{$\{ z\in P :  z\succ s^\prime\} = \emptyset$}{
    $P \leftarrow P\setminus \{z\in P\mid s^\prime\succeq z\} \cup \{s^\prime\}$\;}
    }
    \caption{GSEMO}\label{alg:GSEMO}
    
\end{algorithm}

\section{Theoretical Analysis}\label{sec:theoretical}

In this section, we theoretically investigate the performance of RLS\_swap, GSEMO and (1+1)~EA on different versions of the PWT problem by using the running time analysis. We first consider RLS\_swap and GSEMO on the PWT with correlated weights and profits. Next, the behaviour of (1+1)~EA on the PWT problem with uniform weights is analysed.

To study the PWT problem, we need to investigate the properties of an optimal solution and the impact of adding or removing an item on the benefit function. For this reason, in the following lemma we prove that the weight of the current solution, $w_i$, and $p_i$ determine if item $e_i$ is worth adding to or removing from the current solution.
\begin{lemma}\label{lem:uniq_weight}
	For each item $e_i$ there is a unique threshold $w_{e_i}$ such that adding $e_i$ to the current solution $s$ improves the benefit function if and only if $W(s)< w_{e_i}$. Moreover, $W(s)>w_{e_i}+w_i$ if and only if removing $e_i$ increases the benefit function.
\end{lemma}
\begin{proof}
	Assume that the current solution $s$ does not include item $e_i$ and $s^\prime = s \cup {e_i}$. Hence, we have
	\begin{align}\nonumber
	B(s^\prime)-B(s) &= (P(s^\prime)-RT(s^\prime))-(P(s)-RT(s))\\\nonumber
	& = p_i - Rd\left(\frac{1}{v_{\max}-\nu (W(s)+w_i)}-\frac{1}{v_{\max}-\nu W(s)}\right)\\
	& = p_i - \frac{Rd\nu w_i}{(v_{\max}-\nu (W(s)+w_i))\cdot(v_{\max}-\nu W(s))}\label{equ:w_e_i} \text{.}
	\end{align}
	Adding $e_i$ to $s$ increases $B(s)$ only if the value of Expression~\ref{equ:w_e_i} is greater than zero. Solving this equation we have
	\begin{align}\label{equ:w_e_i1}
	B(s^\prime)-B(s)\geq0 \iff W(s)\leq\overbrace{ \frac{v_{\max}}{\nu}-\frac{w_i}{2}\left(1+\sqrt{1+\frac{4Rd}{\nu w_i p_i}}\right)}^{\mathlarger{w_{e_i}}}\text{.}
	\end{align}
	By solving Equation~\ref{equ:w_e_i} for each $e_i\,\text{,}\; 1\leq i \leq n$, we can find $w_{e_i}$ such that adding $e_i$ to $s$ improves $B(s)$ if and only if $W(s)<w_{e_i}$. Considering the case that $e_i\in s$ and $s^\prime = s\setminus\{e_i\}$, a similar calculation implies that
	\begin{align}\label{equ:w_e_i2}
	B(s^\prime)-B(s)\geq0 \iff W(s)&\geq \frac{v_{\max}}{\nu}+\frac{w_i}{2}\left(1-\sqrt{1+\frac{4Rd}{\nu w_i p_i}}\right)\\\nonumber
	&\geq w_{e_i}+w_i\text{.}
	\end{align}
	In other words, removing $e_i$ from $s$ increases the benefit function if and only if $W(s)>w_{e_i}+w_i$, which completes the proof. 
\end{proof}
A direct result from Equations~\ref{equ:w_e_i1} and~\ref{equ:w_e_i2} is that
$$p_i\geq p_j \land w_i\leq w_j  \implies w_{e_i}\geq w_{e_j}\text{.}$$ Therefore, we have $w_{e_1}> w_{e_2}> \cdots> w_{e_n}$. Moreover, if the weight of solution $s$ equals $w_{e_i}$ then $B(s)=B(s\cup e_i)$. Note that for any item $e_i$ and $e_j$, $i<j$, if removing $e_i$ is beneficial, then it is the same for $e_j$ because of the correlated weights and profits. Hence, for any $1\leq i<j\leq n$, we also have
\begin{align}\label{equ:removing_weights}
    w_{e_i}+w_i> w_{e_j}+ w_j\text{.}
\end{align}

Now we discuss the optimal solution of the PWT problem. Let $s_i = <1^i,0^{n-i}>$ denote the solution that only includes the first $i$ items and $s_0 = <0^n>$. Consider the case that for some $i$, we have $w_{e_i}>W(s_{i-1})$. Hence, by Lemma~\ref{lem:uniq_weight}, adding $e_i$ to $s_{i-1}$ improves the benefit function and $B(s_i)>B(s_{i-1})$. Moreover, since $w_{e_1}\geq\cdots\geq w_{e_i}$, we have $B(s_i)\geq\cdots\geq B(s_1)$. The claim is that for a given set of items, there is a unique $k$ such that the optimal solution of the PWT problem is packing either $k$ or $k \land (k+1)$ items with the lowest indices. In the second case, $s_k$ and $s_{k+1}$ are both optimal and have the same benefit value. The following lemma proves this claim. 

\begin{lemma}\label{lem:optimalsolution}
	The optimal solution for the Packing While Travelling problem is the set of $k$ or $k \land (k+1)$ items with highest profits and lowest weights where $k$ is unique and depends on the given set of items.
\end{lemma}
\begin{proof}
	First, we assume that $w_{e_n}> W(s_{n-1})$. In this case, we have $B(s_{n})>B(s_{n-1})$ and $s_n$ is the optimal solution. On the other hand, if $w_{e_1}<0$ then the optimal solution is $s_0$. In the rest of the proof, we assume that neither of these cases happen.
	
	Let $o=\min\{i\mid W(s_{i})>w_{e_{(i+1)}}, {0\leq i\leq n-1}\}$. According to the definition of $o$ and $w_{e_{o+1}}$, adding any item to $s_o$ decreases the benefit function. Moreover, we have $W(s_{o-1})<w_{e_o}$, which implies $W(s_{o-1})+w_o=W(s_o)<w_{e_o}+w_o$. Hence, removing any item from $s_o$ also reduces the benefit function. Therefore, the following equation holds:
	$$B(s_0)<\cdots<B(s_o)\geq B(s_{o+1})>\cdots>B(s_n)\text{.}$$ The equality $B(s_o)=B(s_{o+1})$ holds only when $w_{e_{o+1}}=W(s_o)$. However, all other inequalities are strict according to Lemma~\ref{lem:uniq_weight}.
	
	To prove that $s_o$ or $s_o \land s_{o+1}$ are the only optimums, it is now enough to show that for any $i$, $s_i$ has the highest benefit value among other solutions with $i$ items. This is also true since the items in $s_i$ have the highest profits and the lowest weights, which result in the highest benefit value.
	
	Involving the capacity constraint $C$, however, may change the optimal solution. We define $s_k$, $k=\max\{j\mid W(s_{j})\leq C, {0\leq j\leq o}\}$, the feasible solution with the highest benefit function, which is the actual optimal solution with respect to $C$. This finalises the proof. 
\end{proof}
From this point, we denote the optimal solution by $s^*$ and $k=|s^*|$ denotes the number of selected items in the optimal solution.

\subsection{Correlated Weights and Profits}\label{sub:Correlated}
In this section, we consider the instances of the PWT problem in which the weights are strongly correlated with the profits. We calculate the performance of RLS\_swap and GSEMO for this type of PWT.
\subsubsection{RLS\_swap}

Using the result of Lemma~\ref{lem:optimalsolution}, we analyse the performance of RLS\_swap finding the optimal solution of the PWT problem in terms of the number of evaluations. We refer to the first $k$ bits of $s^*$ as the first block and the rest as the second block. Let $l$ and $r$ denote the number of zeros in the first block and the number of ones in the second block, respectively. For technical reason, we assume item $e_0$ exists where $w_{e_0} > \sum_{i=1}^n w_i$ and $x_0 = 1$. We denote the solution achieved by RLS\_swap after $t$ generations as $s^t$. Consequently, we define $$h_t =\max\{0\leq i\leq k \mid x_0=x_1=\cdots= x_{i}=1 \land W(s^t)< w_{e_i}+w_i\}\text{,}$$ to be the index of a specific bit of $s^t$.
The following lemma and theorem consider the performance of RLS\_swap on the PWT problem with correlated weights.

\begin{lemma}\label{lem:2ndphase}
Having obtained a solution $s^t$, RLS\_swap does not accept a solution $s^\prime=<x^{\prime}_1,\cdots,x^{\prime}_n>$ in which $\exists i\leq h_t : x^{\prime}_i = 0$.
\end{lemma}
\begin{proof}
    Since the weights and profits are correlated, it is enough to prove that RLS\_swap does not remove $e_{h_t}$. Let $s^\prime$ denote the solution in generation $t'>t$ such that $W(s^\prime)\geq w_{e_{h_t}}+w_{h_t}$ for the first time after $t$, i.e. RLS\_swap is able to remove $e_{h_t}$ from $s^\prime$.
    Note that all the items $e_i$, $i\leq h_t$, are still in $s^\prime$ at time $t'$. Furthermore, no swap mutation can remove $e_{h_t}$ since it has a higher profit value and less weight than other missed items. Hence, there exists an item $e_x$, $x>h_t$, that has been added to $s^\prime$ with a one-bit flip in iteration $t'-1$ such that $W(s^\prime)- w_x< w_{e_{h_t}}+w_{h_t}$ , $C\geq W(s^\prime)\geq w_{e_{h_t}}+w_{h_t}$ and $B(s^\prime)\geq B(s^\prime \setminus e_x)\text{.}$ From the benefit inequality and Lemma~\ref{lem:uniq_weight}, we have $$W(s^{\prime})-w_x\leq w_{e_x} \Rightarrow W(s^\prime )\leq w_{e_x}+w_x\text{.}$$ Since $x>h_t$, according to Inequality~\ref{equ:removing_weights}, we have $w_{e_{h_t}}+w_{h_t}> w_{e_x}+ w_x$. Hence, $W(s')<w_{e_{h_t}}+w_{h_t}$ which contradicts with the assumption that it is beneficial to remove $e_{h_t}$ from $s'$ and completes the proof.
\end{proof}

According to the definition of $h_t$, the weight of the accepted solutions after generation $t$ is lower bounded by $W(s_{h_t})$, in which the first $h_t$ items are selected. This shows that the maximum possible value for $h_t$ is $k$, otherwise $s^*$ is not the optimal solution (Lemma~\ref{lem:optimalsolution}). Let 
\begin{align}\label{def:y}
y=\max\left\{h_t,\min\{ i\mid h_t< i\leq n\land  x_{i}=1\}\right\}\text{,}    
\end{align} be the index of the first one bit after a sequence of zeros after $h_t$. Let $y=h_t$ if there is no one bit after $h_t$. Similarly to the proof of Lemma~\ref{lem:2ndphase}, the following corollary holds.

\begin{corollary}\label{cor:extenstion}
     Obtaining a solution $s$ with $W(s)<w_{e_y}+w_y$, RLS\_swap does not remove $e_y$ except by reducing the value of $y$.
\end{corollary}
Reducing the value of $y$ to $y'<y$ is either caused by swapping $e_y$ with item $e_{y'}$ or inserting $e_{y'}$ with a one-bit flip. In the first case, due to Inequality~\ref{equ:removing_weights}, and in the second case, due to Lemma~\ref{lem:uniq_weight}, we have $W(s')<w_{e_{y'}}+w_{y'}$, i. e. Corollary~\ref{cor:extenstion} holds for $y'$.
\begin{theorem}\label{thm:RLSswap}
RLS\_swap finds the optimal solution for the Packing While Travelling problem with correlated weights and profits in $O(n^3)$ expected time.
\end{theorem}
\begin{proof}

    Let RLS\_swap start with a random solution $s^{0}$. We analyse the optimisation process in three main phases. In the first phase, RLS\_swap finds a feasible solution. The second phase is to obtain $h_t = k$ and the third phase is to remove the remaining items from the second block to achieve the optimal solution.
    
    If $W(s^0)\leq C$ the first phase is already complete. Therefore, let $W(s^0)> C$. Using a fitness level argument,  Neumann and Sutton proved that (1+1)~EA, which uses a fitness function with strictly higher priority in weight constraint satisfaction, finds a feasible solution for KP with correlated weights and profits in $O(n^2)$ expected time (Theorem 3 in~\cite{DBLP:conf/ppsn/0001S18}). Their proof also holds for RLS\_swap since the constraint is linear and RLS\_swap is able to do a one-bit flip in $O(n)$ expected time. Hence, RLS\_swap finds a feasible solution $s$ in $O(n^2)$ expected time and completes the first phase.
    
    In the second phase, we analyse the expected time needed to increase the value of $h_t$ to $h_t=k$. To do this, we need to calculate the expected time to find a solution $s$ such that $x_{h_t+1} =1$ and $W(s)\leq w_{e_{h_t+1}} + w_{h_t+1}$. Let $s$ denote the current solution. According to Lemma~\ref{lem:uniq_weight} and~\ref{lem:optimalsolution}, if $W(s)\geq w_{e_{k+1}}+w_{k+1}$, removing any item from the second block improves the benefit function. Moreover, adding any item to the second block decreases the benefit. The probability of a one-bit flip, which removes an item from the second block, is $r/(2n)$ and it happens in expected time $O(2n/r)$. Since there are at most $n$ items to be removed, RLS\_swap obtains $s$ with $W(s)<w_{e_{k+1}}+w_{k+1}$ in $2n\left(\sum_{i=1}^r 1/i\right) = O(n\log{n})$ expected time. Note that after this point, no item can be removed from $s$ with a one-bit flip.
    
    Now we examine the value of $y$ as defined in Equation~\ref{def:y}. Let $y>h_t$. Since there is no accepted one-bit flip that removes an item, we have $W(s)<w_{e_y}+w_y$ and Corollary~\ref{cor:extenstion} holds. Otherwise, we have $y=h_t$ and there are two cases: $y=h_t=k$ or $y=h_t<k$. Let $y=h_t=k$. In this case, $s$ includes of all the items in the first block and there is no item in the second block to be removed. Therefore, RLS\_swap has found $s^*$ which is the optimal solution. In the other case, that $y=h_t<k$, adding $e_{h_t+1}$ is beneficial and there exists at least one acceptable one-bit flip that adds an item to the solution $s$ in expected time $O(n)$. Let $y>h_t$ be the first item added. Since this bit flip is beneficial, we have $W(s)+w_y<w_{e_y}+w_y$ and Corollary~\ref{cor:extenstion} holds. Thus in $O(n)$ we have a solution in which $y>h_t$ and RLS\_swap cannot remove $e_y$. At this stage, any two-bit flip that swaps $e_y$ and $e_{h_t+1}$ improves $h_t$ by one. This swap takes place in $O(n^2)$ expected time. Thus in expected time $O(n^2)$ RLS\_swap increases $h_t$ by one. Since the maximum value of $h_t$ is $k\leq n$, RLS\_swap achieves $h_t = k$ in expected time $O(n^3)$.

    Finally in the third phase, RLS\_swap needs to remove the remaining items from the second block. Note that the first $k$ items are now selected and cannot be removed anymore. Each item can be removed with a one-bit flip which results in $O(n\log n)$ expected time for this phase.
    
    Therefore we can conclude that RLS\_swap finds the optimal solution for the PWT with correlated weights and profits in $O(n^3)$ expected time.
\end{proof}

\subsubsection{GSEMO}
In this section we consider the time performance of GSEMO on the PWT problem with correlated weights and profits. The two objectives used in this algorithm are the weight function and the lexicographical fitness function, denoted by $W(s)$ and $F(s)$, respectively. We say solution $s_1$ weakly dominates solution $s_2$, denoted by $s_1\succeq s_2$, if $W(s_1)\leq W(s_2)\land F(s_1)\geq F(s_2)$. In the case of at least one strict inequality, it is called strong dominance. In our analysis, which is inspired by~\cite{DBLP:conf/ppsn/0001S18}, we use a fitness level argument on the weights of the solution and compute the expected time needed to find at least one Pareto solution. Next, we calculate the time for finding the entire Pareto front. Due to Lemma~\ref{lem:optimalsolution}, we can observe the following corollary that describes the Pareto front structure.

\begin{corollary}\label{cor:Pareto_size}
    The Pareto set corresponding to the Packing While Travelling problem with correlated weights and profits is the solution set $\{s_0,\cdots, s_k=s^*\}$.
\end{corollary}
\begin{proof}
As it is explained in the proof of Lemma~\ref{lem:uniq_weight}, $s_i$, $i\leq n$, dominates all the solutions with size $i$. On the other hand, we have $B(s_0)\leq\cdots\leq B(s_k)$ while $W(s_0)\leq\cdots\leq W(s_k)$, which implies that $\{s_0,\cdots,s_k\}$ do not dominate each other. Furthermore, $s_k$ dominates every solution $s_i$, $i> k$, since it has a higher fitness value and less weight. Thus, there exists no solution dominating $\{s_0,\cdots, s_k=s^*\}$, which completes the proof.
\end{proof}

The following theorem proves that the expected time for GSEMO to find the entire Pareto front of the PWT problem with correlated weights and profits is $O(n^3)$.

\begin{theorem}
    GSEMO finds all the non-dominated solutions of the Packing While Travelling problem with correlated weights and profits in $O(n^3)$ expected time.
\end{theorem}
\begin{proof}
    The proof consists of two phases. In the first phase, GSEMO finds the Pareto solution $\{0\}^n$. In the second phase, GSEMO finds other Pareto solutions based on the assumption of having at least one optimal solution.
    
    Note that according to the definition of $F(s)$, an infeasible solution with less constraint violation always dominates the other ones, even in this two-objective space. Hence, while GSEMO has not achieved a feasible solution, the size of its population remains one. Therefore, it behaves exactly the same as (1+1)~EA until it finds a feasible solution, which is argued in first phase of Theorem~\ref{thm:RLSswap}, and takes $O(n^2)$ expected time. In the rest of the analysis of the first phase, we assume that GSEMO has found a feasible solution.
    
    To analyse the first phase, we define $n$ fitness levels $A_i$, $0\leq i\leq n-1$, based on the weight objective $W(s)$ as follows:
    \begin{align}\nonumber
        &A_i = \left\{s\mid s=s_0\right\}&&i=0\\\nonumber
        &A_i = \left\{s\mid W(s_{i-1})< W(s)\leq W(s_i)\right\}&&1\leq i\leq n\text{.}
    \end{align}
    According to the correlation of the weights, this definition guarantees that if a solution $s$ belongs to level $A_i$, $1\leq i\leq n$, then $s$ includes at least one of the items $e_j$ where $j\geq i$. Moreover, removing $e_j$ from $s$ moves it to a lower level. Now assume that $P$ denotes the population in step $t$ of GSEMO. Let $A_u$ denote the lowest level that has been achieved until step $t$ and $s\in A_u$ has the highest fitness among the solutions in level $A_u$. GSEMO chooses this solution by choosing $|s|_1$ from $I$, which happens with the probability of $1/|I|\geq 1/n$. Furthermore, there exists item $e_x$ in $s$ such that $x \geq u$ and removing it produces a solution in a lower level. This solution will be accepted since it has the lowest weight. The one-bit flip that removes $e_x$ from $s$ happens with the probability of $1/(en)$. Thus, GSEMO reduces $u$ with probability of $1/(en^2)$ in step $t$. In other words, GSEMO reduces $u$, at least by one, in expected time $O(n^2)$. On the other hand, $s_0$, which is the only solution of level $A_0$, is a Pareto solution since it has the lowest possible weight. Thus, if the algorithm achieves $A_0$ then the first phase is accomplished. Since the maximum possible value of $u$ is $n$ and $u$ is reduced by one in every $O(n^2)$ expected iterations, GSEMO finishes the first phase in expected time of $O(n^3)$.
    
    In the second phase, we assume $s_i\in P$, $i\leq k$ exists such that $s_i$ is Pareto solution by Corollary~\ref{cor:Pareto_size}, and either $s_{i-1}$ or $s_{i+1}$, which are also Pareto solutions, do not exist in $P$. Otherwise, the second phase is already finished. Adding $e_{i+1}$ to or removing $e_{i}$ from $s_i$ results in $s_{i+1}$ or $s_{i-1}$, respectively. For such a step, the algorithm must choose $i\in I$ and flip the correct bit, which happens with the probability $1/n$ and $1/(en)$, respectively. Hence, the algorithm finds a new Pareto solution from $s_i$ in $O(n^2)$ expected iterations. Based on Corollary~\ref{cor:Pareto_size}, the size of the Pareto set is $k+1\leq n$. Therefore, GSEMO finds all the $k+1$ Pareto solutions of the PWT problem with correlated weights and profits in $O(n^3)$ expected time.
\end{proof}

\subsection{Uniform Weights}\label{sub:uniform}
In this section, we analyse the performance of (1+1)~EA on another version of the Packing While Travelling problem. Here we assume that weights of all the items are one and the profits are arbitrary. Similar to the previous instances, we assume items $\{e_1,\cdots,e_n\}$ are indexed such that $p_1\geq\cdots\geq p_n$. Note that the results of Lemmata~\ref{lem:uniq_weight} and~\ref{lem:optimalsolution} also holds for the uniform weights. Moreover, since correlated weights and profits are actually the more general version of the uniform weights, the results for GSEMO and RLS\_swap also holds.

The analysis of (1+1)~EA is based on the following Lemma which proves the existence of a set of one-bit flips and two-bit flips which transform an arbitrary solution to the optimal solution. 


\begin{lemma}\label{lem:set_of_bitflips}
	Let $W(s^*)$ be the weight of the optimal solution and the current solution $s$ is feasible. There exists a set of one-bit flips and two-bit flips that transform $s$ to an optimal solution if they happen in any order.
	
\end{lemma}
\begin{proof}
	Assume that $W(s)\leq W(s^*)$. This implies that $i=|s|_1\leq k$, where $i$ is the number of selected items in $s$. In this case, any two-bit flips that swap a one bit from the second block and a zero bit from the first block will be accepted by the algorithm. If there is no one bit in the second block, then all the one-bit flips that change a zero bit in the first block are accepted and this set leads $s$ to the optimal solution $s^*$. Note that all defined one-bit flips will be accepted (Lemma~\ref{lem:uniq_weight} and Lemma~\ref{lem:optimalsolution}) since $i\leq k$ and the weights are one. Hence, it is not necessary for the bit flips to happen in a special order. The other case, that $C\geq W(s)>W(s^*)$, is similar to the first one. Any two-bit flips that swap a zero bits in the first block with a one bit in the second block and any one-bit flips that remove ones from the second block are accepted by the algorithm. Thus, in this case, there also exists a set of one-bit flips and two-bit flips that, if they occure in any order, transform $s$ into the optimal solution. 
\end{proof}
Here we present some more definitions that help us with analysis of the performance of (1+1)~EA on the PWT problem.
Assume $M$ is the set of $m_1$ one-bit flips and $m_2$ two-bit flips that transform the current solution $s$ to an optimal solution and $|M|= m_1+m_2$. Moreover, let $g(s)=B(s^*)-B(s)$ be the difference between the benefit of $s$ and the optimal solution. Therefore, we can denote the contribution of one-bit flips and two-bit flips in $g(s)$ by $g_1(s)$ and $g_2(s)$, respectively, such that $g(s) = g_1(s)+g_2(s)$. In the next theorem, we calculate the expected time for (1+1)~EA to find the optimal solution of the PWT problem with unifrom weights.
\begin{theorem}
	The expected time for (1+1)~EA to find the optimal solution for the Packing While Travelling problem with uniform weights is $O(n^2\max\{\log{n},\log {p_{\max}}\})$.
\end{theorem}
\begin{proof}
	As mentioned in the proof of Theorem~\ref{thm:RLSswap}, it is proven in Theorem 3 in~\cite{DBLP:conf/ppsn/0001S18} that (1+1)~EA finds a feasible solution in $O(n^2)$ expected time. Thus, we assume that (1+1)~EA has already found a feasible solution $s$. Let $\Delta_t = g(s)-g(s^\prime) = B(s^\prime)-B(s)$ denote the improvement of (1+1)~EA in iteration $t$ which transforms $s$ to $s^\prime$. We partition the proof into two cases. Firstly, let the overall contribution of one-bit flips in $g(s)$ be more than the total improvement that could be achieved by two-bit flips. Hence, we have $g_1(s)\geq \frac{g(s)}{2}$. Since there are $m_1$ one-bit flips in $M$, each of them happens with the probability of $\frac{m_1}{en}$ and improves the solution in average by $\frac{g_1(s)}{m_1}$. Thus, in this case we have : $$E[\Delta_t]\geq \frac{g_1(s)}{m_1}\cdot\frac{m_1}{en}\geq\frac{g_1(s)}{en}\geq \frac{g(s)}{2en}\text{.}$$
	
	In the second case, the sum of improvements by two-bit flips in $M$ is more than the total improvement of one-bit flips and we have $g_2(s)\geq \frac{g(s)}{2}$. Here, the average improvement of each two-bit flip is $\frac{g_2(s)}{m_2}$ and each takes place with a probability of $\frac{m_2}{2en^2}$. Therefore, we have: $$E[\Delta_t]\geq \frac{g_2(s)}{m_2}\cdot\frac{m_2}{2en^2}\geq\frac{g_2(s)}{2en^2}\geq \frac{g(s)}{4en^2}\text{.}$$
	We can conclude that the expected improvement at step $t$ is at least $E[\Delta_t] \geq\frac{g(s)}{4en^2}$. On the other hand, $B(s^*)\leq n\cdot p_{\max}$ is the maximum possible benefit value, ignoring the cost function. To use the multiplicative drift, it is only necessary to calculate the minimum possible amount of $g(s)$.
	
	Let $s^l$ be the last solution that turned into an optimal solution with a bit flip. We need to find the minimum of $B(s^*)-B(s^l)$ which happens when $B(s^l)$ is maximized. There are three possibilities  in which $s^l$ is the closest solution to the optimal solution. If the last bit flip is a two-bit flip, to maximize $B(s^l)$, $s^l = s_{k-1} \cup e_{k+1}$. In this case, we have $g(s^l) = p_k-p_{k+1}\geq 1$, since all the profits are integers.
	If the last bit flip is a one-bit and it adds an item to $s^l$, then the maximum benefit of $s^l$ is achieved when $s^l=s_{k-1}$. Therefore, we have $g(s^l) =p_k - \frac{Rd\nu}{(v_{\max}-\nu(k-1))(v_{\max}-\nu k)}$. Considering $R,d,v_{\min} \text{ and } v_{\max}$ as constants, $g(s) = O(n^{-q})$ for some constant $q\geq1$. The same result holds for the third case where the final bit flip removes $e_{k+1}$ from $s^l=s_{k+1}$. 
	
	Finally, using the multiplicative drift with $X_0 = n\cdot p_{\max}$, $x_{\min} = n^{-q}$ and $\delta = \frac{1}{4en^2}$, the expected first hitting time $T$ that (1+1)~EA finds the optimal solution for the PWT problem with uniform weights is:
	$$E[T|X_0]\leq 4en^2\ln(n\cdot p_{\max} \cdot n^q)=O(n^2\max\{\log{n},\log {p_{\max}}\})\text{.}$$
\end{proof}

\section{Experiments}\label{sec:practical}

In this section, we experimentally investigate the performance of our algorithms. Our goal is to complement the theoretical analysis to gain additional insight into the behaviour of the algorithms during the optimisation process. Moreover, by analysing the running time of our algorithms on random instances with different sizes, we present a clearer view of how they perform.

 In the previous section, we showed that the classical RLS cannot find the optimal solution by doing only one-bit flips and it is essential to do a two-bit flip to escape local optima. However, what happens if the classical RLS is enhaced with a population? To answer this question, we also consider another two multi-objective algorithms called SEMO and SEMO\_swap. They are different from GSEMO (Algorithm~\ref{alg:GSEMO}) only in the mutation step. The mutation step in SEMO\_swap is the same as Algorithm~\ref{alg:RLSswap}. SEMO, on the other hand, only chooses one bit, uniformly at random, and flips it. Our experiments show that SEMO avoids the classical RLS weakness by using the population and performs even better than the other multi-objective algorithms.

\subsection{Benchmarking and Experimental Setting}
To compare the practical performance of different algorithms in different types of PWT, we used thirty different instances where each instance consists of 300 items. Each instance of size $n$ for the correlated PWT is generated by choosing $n$ integer profits uniformly at random within [1,1000] and assigning them to $p_1,\cdots,p_n$ in descending order. Similarly, $n$ uniformly random integers within the same interval are generated and sorted in ascending order as the weights $w_1,\cdots,w_n$. Hence, for any item $e_i,e_j : i<j$ we have $p_i\geq p_i$ and $w_i\leq w_j$. For the instances with uniform weights, we use the same profits as correlated instances but change the weights to one. We set the constant values in all the instances as follows: $n=300$, $d = 50$, $R=70$, $v_{\max} = 1$ and $v_{\min} = 0.1$. $C=8000$ is the capacity chosen for correlated instances. To calculate the proper capacity for uniform instances, we use the average of the maximum number of items with correlated weights that fit in $C=8000$, which results in $C=72$ for the uniform instances. Furthermore, for all the experiments, algorithms start with the zero solution.

We use these instances in the first experiment, in which we show how the algorithms converge to the optimal solution for the correlated and uniform instances. We run each algorithm for all thirty instances, record the best found solution in each generation and normalise its benefit value to the interval [0,1] with respect to the optimal solution for each instance. We plot the average of normalised values as the success rate for each algorithm (Figure~\ref{fig:convergence}). Hence, this value is one when an algorithm finds the optimal solution for all thirty different instances.
\begin{figure}[t!]
    \begin{subfigure}{0.47\textwidth}
        \centering
        \includegraphics[width=\textwidth]{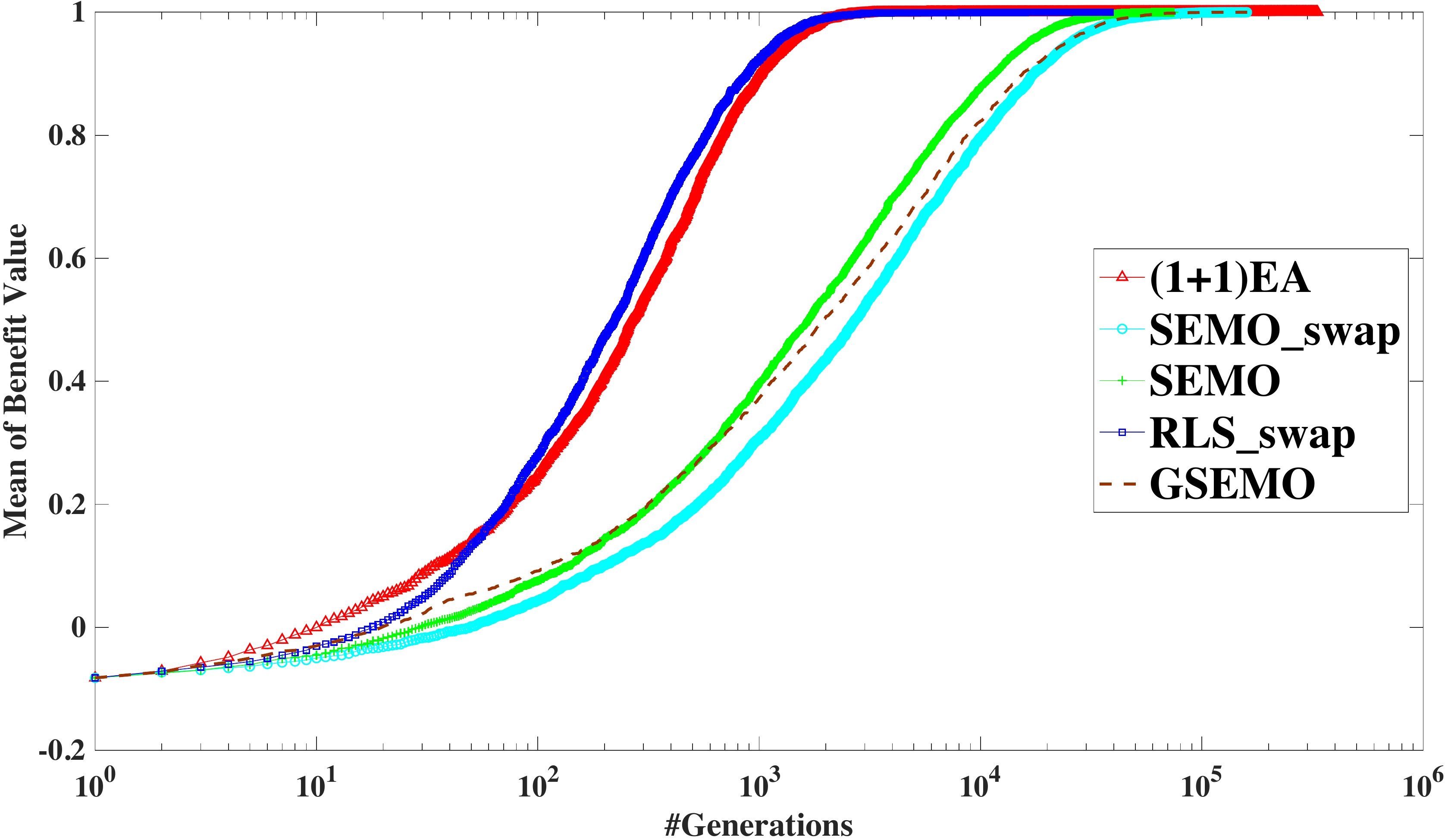}
        
        \caption[]%
        {\small Correlated weights and profits}    
        \label{fig:corr_conv}
    \end{subfigure}
    \begin{subfigure}{0.47\textwidth}  
        \includegraphics[width=\textwidth]{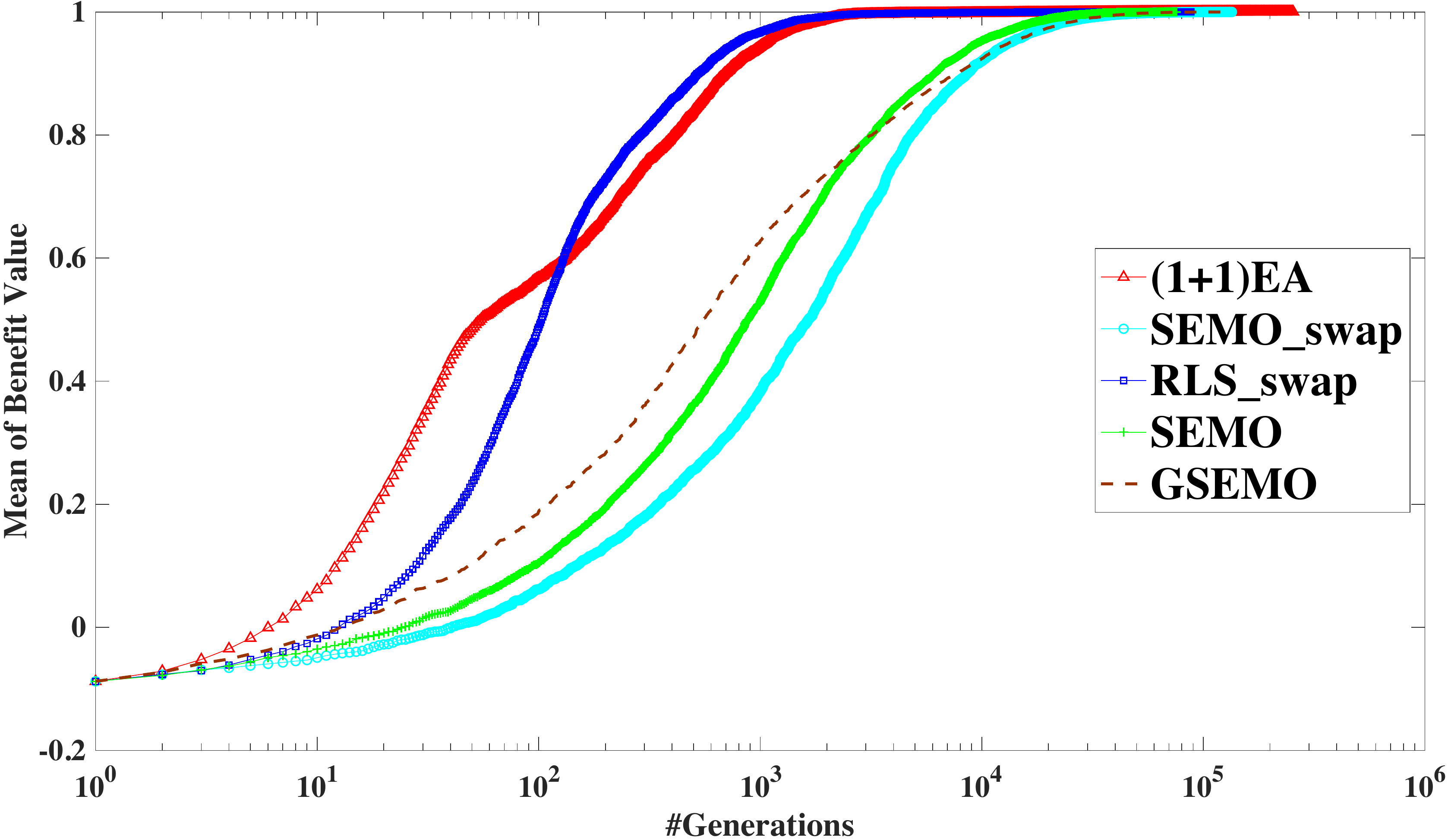}
        
        \caption[]%
        {\small Uniform weights}    
        \label{fig:unif_conv}
    \end{subfigure}
    \caption[]
    {\small The average best solution in each generation for RLS\_swap, (1+1)~EA, GSEMO, SEMO and SEMO\_swap in thirty different instances.} 
        \label{fig:convergence}
    \vspace{-0.15cm}
\end{figure}

In the last experiments, in which we consider the optimisation time for algorithms, we run each algorithm on instances with seven different sizes, $n=100,200,500,1000,2000,5000,10000$. For each $n$, we have thirty different instances which are created with the same constants as the previous ones. For each algorithm, we record the average number of evaluations to find the optimal solution for the thirty instances with the same $n$ as its optimisation time. Hence, for each algorithm, we have seven points to estimate the performance (Figure~\ref{fig:all_times}). We only do this analysis with the correlated instances.

In our results for (1+1)~EA and GSEMO, we only count the generations that an actual bit flip has happened in mutation steps. The same approach has been used in~\cite{DBLP:conf/gecco/Doerr018}. This causes the analysis to be fair since RLS\_swap guarantees to create a new solution in each mutation step but (1+1)~EA and GSEMO do not flip any of the bits with the probability of $1/e$ (36.7\% of times).
\subsection{Analysis}
In this section, we analyse the performance of the algorithms based on the experimental results. Figure~\ref{fig:convergence} illustrates how the algorithms converge to the optimal solutions for all the instances on average. In both types of the instances, it can be observed that (1+1)~EA and RLS get close to the optimal solutions much faster than the multi-objective algorithms, which demonstrates that using population slows down the convergence rate significantly. (1+1)~EA, however, has major problems in finding the optimal solutions when it is close to the optimum and only a few bit flips are needed. The reason is that (1+1)~EA is able to flip more than one bit. Hence, it is likely to improve the benefit value, but increase the hamming distance between the current solution and the optimal solution in the same time. In other words, (1+1)~EA is able to improve the benefit function while the improved solution needs more bit flips to reach the optimal solution than before. Note that Figure~\ref{fig:convergence} is plotted with a logarithmic x scale to make it easier to distinguish the difference betweem the first $10^5$ generations. While the other algorithms almost achieve the optimal solution in the first $10^5$ generations, (1+1)~EA performs two times worse than the others.

Looking into the multi-objective algorithms, Figure~\ref{fig:convergence} illustrates that GSEMO and SEMO\_swap behave almost identically while SEMO outperforms both. This shows that although SEMO can flip only one bit in each generation, the population enables the algorithm to bypass the local optima and find a better solution. In other words, using population and dominance concepts is another method to escape local optima with only one-bit flips. The reason, as described in the proof of Theorem 3.7, is that one-bit flips can always reduce the weight of the solution. Using weight as an objective in population-based multi-objective algorithms makes it possible for SEMO to obtain the zero solution, which is a Pareto optimal solution, by using only one-bit flips. From that point, the algorithm can gradually find all the other Pareto optimal solutions, including the general optimal solution, by adding the correct item to the previous optimal solution.
Generally speaking, RLS\_swap and SEMO, which enhance versions of the classical RLS to avoid the local optima, perform better for this Packing While Travelling problem variation.
\begin{figure}[t!]
        \centering
        \includegraphics[width=0.47\textwidth]{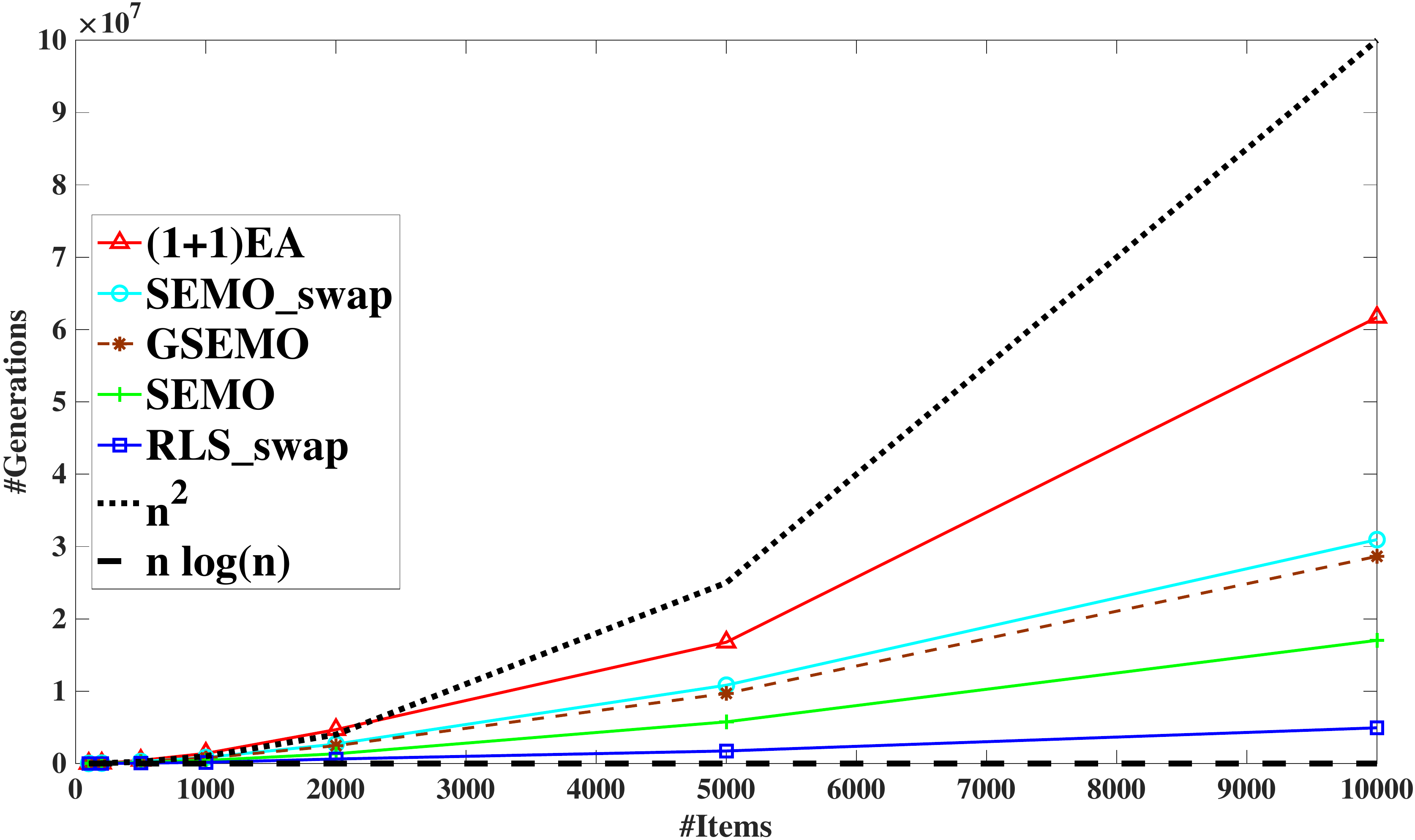}
        \caption[]%
        {\small Comparison between the running time of the algorithms according to the experimental results.}    
        \label{fig:all_times}
\end{figure}

The results of the final experiment are presented in Figure~\ref{fig:all_times}. We run each algorithm for instances with seven different sizes, and for each specific size, we have thirty different instances. Figure~\ref{fig:all_times} shows the average of running times for each instance size. It validates the results in Figure~\ref{fig:convergence}, in which RLS\_swap has a better performance than the other algorithms while multi-objective algorithms outperform (1+1)~EA. This data can also be used to give an insight into the expected optimisation time for each algorithm. Figure~\ref{fig:all_times} presents the data for $n^2$ and $n\log{n}$ expressions which suggests that the order of magnitude for the running times for these algorithms on random instances is close to $n^2$ or $n\log{n}$.

\section{Conclusion}\label{sec:conclusion}
Evolutionary algorithms are known as general problem solvers which can provide good quality solutions to many problems. While these algorithms have been thoroughly studied in solving linear functions, their performance on non-linear problems is not clear. In this paper, we study the performance of three base-line EAs, from a theoretical viewpoint,  on the Packing While Travelling problem, which is also known as a non-linear knapsack problem. We prove that RLS\_swap and GSEMO find the optimal solution in $O(n^3)$ expected time for instances with correlated weights and profits. In addition, we show that (1+1)~EA finds the optimal solution for instances with uniform weights in $O(n^2\log(\max\{n,p_{\max}\}))$, where $p_{\max}$ is the highest profit of the given items. We also empirically investigate these algorithms and, based on our investigations, we conjecture an upper bound of $O(n^2)$.
We aim to extend our results for the PWT with more than two cities with correlated items and items with uniform weights. However, the case in which the items are not correlated and are correlated in each city separately might be as hard as the general case.

\section*{Acknowledgements}
This work has been supported by the Australian Research Council through grant DP160102401. Furthermore, Frank Neumann has been supported by the Alexander von Humboldt Foundation through a Humboldt Fellowship for Experienced Researchers.

\bibliographystyle{unsrt}  
\bibliography{references}  

\begin{thebibliography}{10}

\bibitem{DBLP:books/sp/chiong12}
Raymond Chiong, Thomas Weise, and Zbigniew Michalewicz, editors.
\newblock {\em Variants of Evolutionary Algorithms for Real-World
  Applications}.
\newblock Springer, 2012.

\bibitem{DBLP:journals/gpem/Popovici14}
Elena Popovici.
\newblock Anne auger and benjamin doerr (eds): Theory of randomized search
  heuristics: foundations and recent developments - world scientific (2011),
  359 pp.
\newblock {\em Genetic Programming and Evolvable Machines}, 15(1):111--112,
  2014.

\bibitem{DBLP:series/ncs/Jansen13}
Thomas Jansen.
\newblock {\em Analyzing Evolutionary Algorithms - The Computer Science
  Perspective}.
\newblock Natural Computing Series. Springer, 2013.

\bibitem{neumann2010bioinspired}
Frank Neumann and Carsten Witt.
\newblock {\em Bioinspired computation in combinatorial optimization:
  Algorithms and their computational complexity}.
\newblock Springer Science \& Business Media, 2010.

\bibitem{droste2002analysis}
Stefan Droste, Thomas Jansen, and Ingo Wegener.
\newblock On the analysis of the (1+ 1) evolutionary algorithm.
\newblock {\em Theoretical Computer Science}, 276(1-2):51--81, 2002.

\bibitem{DBLP:journals/ec/DrosteJW98}
Stefan Droste, Thomas Jansen, and Ingo Wegener.
\newblock A rigorous complexity analysis of the {(1} + 1) evolutionary
  algorithm for separable functions with boolean inputs.
\newblock {\em Evolutionary Computation}, 6(2):185--196, 1998.

\bibitem{DBLP:journals/tcs/DoerrK15}
Benjamin Doerr and Marvin K{\"{u}}nnemann.
\newblock Optimizing linear functions with the (1+{\(\lambda\)}) evolutionary
  algorithm - different asymptotic runtimes for different instances.
\newblock {\em Theor. Comput. Sci.}, 561:3--23, 2015.

\bibitem{Witt:2018:DCW:3205455.3205581}
Carsten Witt.
\newblock Domino convergence: Why one should hill-climb on linear functions.
\newblock In {\em Proceedings of the Genetic and Evolutionary Computation
  Conference}, GECCO '18, pages 1539--1546, 2018.

\bibitem{friedrich2018analysis}
Tobias Friedrich, Timo K{\"o}tzing, JA~Gregor Lagodzinski, Frank Neumann, and
  Martin Schirneck.
\newblock Analysis of the (1+ 1) ea on subclasses of linear functions under
  uniform and linear constraints.
\newblock {\em Theoretical Computer Science}, 2018.

\bibitem{FrankGECCO}
Junhua Wu, Sergey Polyakovskiy, and Frank Neumann.
\newblock Improved runtime results for simple randomised search heuristics on
  linear functions with a uniform constraint.
\newblock In {\em Genetic and Evolutionary Computation Conference (GECCO '19),
  July 13-17, 2019, Prague, Czech Republic}. {ACM}, 2019.

\bibitem{DBLP:conf/ppsn/AntipovD18}
Denis Antipov and Benjamin Doerr.
\newblock Precise runtime analysis for plateaus.
\newblock In Anne Auger, Carlos~M. Fonseca, Nuno Louren{\c{c}}o, Penousal
  Machado, Lu{\'{\i}}s Paquete, and Darrell Whitley, editors, {\em Parallel
  Problem Solving from Nature - {PPSN} {XV} - 15th International Conference,
  Coimbra, Portugal, September 8-12, 2018, Proceedings, Part {II}}, volume
  11102 of {\em Lecture Notes in Computer Science}, pages 117--128. Springer,
  2018.

\bibitem{DBLP:journals/tec/NijssenB03}
Siegfried Nijssen and Thomas B{\"{a}}ck.
\newblock An analysis of the behavior of simplified evolutionary algorithms on
  trap functions.
\newblock {\em {IEEE} Trans. Evolutionary Computation}, 7(1):11--22, 2003.

\bibitem{giessen2016robustness}
Christian Gie{\ss}en and Timo K{\"o}tzing.
\newblock Robustness of populations in stochastic environments.
\newblock {\em Algorithmica}, 75(3):462--489, 2016.

\bibitem{Bonyadi2019}
Mohammad~Reza Bonyadi, Zbigniew Michalewicz, Markus Wagner, and Frank Neumann.
\newblock {\em Evolutionary Computation for Multicomponent Problems:
  Opportunities and Future Directions}, pages 13--30.
\newblock Springer International Publishing, Cham, 2019.

\bibitem{corus2016parameterised}
Dogan Corus, Per~Kristian Lehre, Frank Neumann, and Mojgan Pourhassan.
\newblock A parameterised complexity analysis of bi-level optimisation with
  evolutionary algorithms.
\newblock {\em Evolutionary computation}, 24(1):183--203, 2016.

\bibitem{pourhassan2018theoretical}
Mojgan Pourhassan and Frank Neumann.
\newblock Theoretical analysis of local search and simple evolutionary
  algorithms for the generalized travelling salesperson problem.
\newblock {\em Evolutionary computation}, pages 1--34, 2018.

\bibitem{DBLP:conf/cec/BonyadiMB13}
Mohammad~Reza Bonyadi, Zbigniew Michalewicz, and Luigi Barone.
\newblock The travelling thief problem: The first step in the transition from
  theoretical problems to realistic problems.
\newblock In {\em Proceedings of the {IEEE} Congress on Evolutionary
  Computation, {CEC} 2013, Cancun, Mexico, June 20-23, 2013}, pages 1037--1044.
  {IEEE}, 2013.

\bibitem{DBLP:journals/eor/PolyakovskiyN17}
Sergey Polyakovskiy and Frank Neumann.
\newblock The packing while traveling problem.
\newblock {\em European Journal of Operational Research}, 258(2):424--439,
  2017.

\bibitem{DBLP:conf/algocloud/NeumannPSSW18}
Frank Neumann, Sergey Polyakovskiy, Martin Skutella, Leen Stougie, and Junhua
  Wu.
\newblock A fully polynomial time approximation scheme for packing while
  traveling.
\newblock In Yann Disser and Vassilios~S. Verykios, editors, {\em Algorithmic
  Aspects of Cloud Computing - 4th International Symposium, {ALGOCLOUD} 2018,
  Helsinki, Finland, August 20-21, 2018, Revised Selected Papers}, volume 11409
  of {\em Lecture Notes in Computer Science}, pages 59--72. Springer, 2018.

\bibitem{DBLP:conf/cpaior/PolyakovskiyN15}
Sergey Polyakovskiy and Frank Neumann.
\newblock Packing while traveling: Mixed integer programming for a class of
  nonlinear knapsack problems.
\newblock In Laurent Michel, editor, {\em Integration of {AI} and {OR}
  Techniques in Constraint Programming - 12th International Conference,
  {CPAIOR} 2015, Barcelona, Spain, May 18-22, 2015, Proceedings}, volume 9075
  of {\em Lecture Notes in Computer Science}, pages 332--346. Springer, 2015.

\bibitem{DBLP:conf/gecco/WuPN16}
Junhua Wu, Sergey Polyakovskiy, and Frank Neumann.
\newblock On the impact of the renting rate for the unconstrained nonlinear
  knapsack problem.
\newblock In Tobias Friedrich, Frank Neumann, and Andrew~M. Sutton, editors,
  {\em Proceedings of the 2016 on Genetic and Evolutionary Computation
  Conference, Denver, CO, USA, July 20 - 24, 2016}, pages 413--419. {ACM},
  2016.

\bibitem{DBLP:journals/tcs/NeumannW07}
Frank Neumann and Ingo Wegener.
\newblock Randomized local search, evolutionary algorithms, and the minimum
  spanning tree problem.
\newblock {\em Theor. Comput. Sci.}, 378(1):32--40, 2007.

\bibitem{DBLP:conf/ppsn/0001S18}
Frank Neumann and Andrew~M. Sutton.
\newblock Runtime analysis of evolutionary algorithms for the knapsack problem
  with favorably correlated weights.
\newblock In Anne Auger, Carlos~M. Fonseca, Nuno Louren{\c{c}}o, Penousal
  Machado, Lu{\'{\i}}s Paquete, and Darrell Whitley, editors, {\em Parallel
  Problem Solving from Nature - {PPSN} {XV} - 15th International Conference,
  Coimbra, Portugal, September 8-12, 2018, Proceedings, Part {II}}, volume
  11102 of {\em Lecture Notes in Computer Science}, pages 141--152. Springer,
  2018.

\bibitem{DBLP:conf/gecco/Doerr018}
Carola Doerr and Markus Wagner.
\newblock Simple on-the-fly parameter selection mechanisms for two classical
  discrete black-box optimization benchmark problems.
\newblock In Hern{\'{a}}n~E. Aguirre and Keiki Takadama, editors, {\em
  Proceedings of the Genetic and Evolutionary Computation Conference, {GECCO}
  2018, Kyoto, Japan, July 15-19, 2018}, pages 943--950. {ACM}, 2018.

\end{thebibliography}






\end{document}